%% file: main-arxiv.tex
\documentclass{article}

\usepackage[left=1in, right=1in, top=1in, bottom=1in]{geometry}

\usepackage[utf8]{inputenc} 
\usepackage[T1]{fontenc}    
\usepackage{hyperref}       
\usepackage{url}            
\usepackage{booktabs}       
\usepackage{amsfonts}       
\usepackage{nicefrac}       
\usepackage{microtype}      
\usepackage{xcolor}         

\usepackage{microtype}
\usepackage{graphicx}
\usepackage{subcaption}
\usepackage{booktabs} 

\usepackage{hyperref}

\usepackage{amsmath}
\usepackage{amssymb}
\usepackage{mathtools}
\usepackage{amsthm}

\usepackage{thm-restate}
\usepackage{array}
\usepackage[capitalize,noabbrev]{cleveref}

\theoremstyle{plain}
\newtheorem{theorem}{Theorem}[section]
\newtheorem{proposition}[theorem]{Proposition}

\theoremstyle{definition}

\theoremstyle{remark}
\newtheorem{remark}[theorem]{Remark}

\usepackage[textsize=tiny]{todonotes}

\input{math_commands}

\NewDocumentCommand{\TT}{s m}{\textcolor{cyan}{\IfBooleanF{#1}{\textbf{Taos:}~}#2}}

\usepackage{multirow}

\usepackage{verbatim}
\usepackage{enumitem}

\usepackage{wrapfig}

\title{\Huge\bf Second Order Drifting Models}

\author{
Drake Brown*, Yuhao Huang*, Shih-Hsin Wang, Bao Wang \\
  Department of Mathematics\\
  Scientific Computing and Imaging (SCI) Institute\\
  University of Utah\\ \footnote{Drake Brown and Yuhao Huang are co-first authors.}
}

\begin{document}

\maketitle

\begin{abstract}
Drifting models are a recent class of one-step generative models that evolve the model distribution during training using a predefined sample-based drift field. Although they avoid iterative inference, their kernel-based drift fields induce frequency-dependent training dynamics: In the linearized regime, each Fourier mode of the density residual decays at a rate determined by the kernel spectrum, leading to slow recovery of fine-scale structure. We propose Second-Order Drifting Models, which lift drifting dynamics into phase space by augmenting generated samples with artificial velocity variables. We show that the resulting density perturbations obey accelerated second-order dynamics in Fourier space, connecting drifting models to the celebrated Nesterov acceleration from optimization theory. This provides a principled mechanism for mitigating the spectral stiffness of first-order drifting while preserving one-step inference. We derive a practical semi-implicit training algorithm and evaluate it on synthetic distribution matching, sequential data generation, and robotic control. Across these settings, the second-order drifting model improves convergence behavior and achieves competitive or superior performance over first-order drifting baselines. 
\end{abstract}

\section{Introduction}\label{sec:intro}






Flow-based generative models—particularly those based on diffusion (cf.~\cite{sohl2015deep,ho2020denoising,song2021scorebased}) and flow matching (cf.~\cite{lipman2023flow,albergo2023stochastic,liu2023flow,huang2026improving}) mechanisms—can be formulated as learning a transformation $f$ such that the pushforward of a source distribution $q$ matches the target data distribution $p$, i.e., $f_{\sharp} q \approx p$ \footnote{The map $f:\mathcal{X}\to\mathcal{X}$ that transports a source distribution $q$ to the target data distribution $p$. Concretely, if $X \sim q$ and $Y = f(X)$, then $Y$ follows the pushforward distribution $f_{\sharp} q$, defined by
\[
(f_{\sharp} q)(A) := q\bigl(f^{-1}(A)\bigr)
\quad \text{for all measurable sets } A \subseteq \mathcal{X}.
\]
The learning objective is to construct $f$ such that
\(f_{\sharp} q \approx p\).}. 
Despite their remarkable performance in high-fidelity generation \cite{rombach2022high,hoogeboom2022equivariant,watson2023novo,transue2025flow,jia2025plug}, these transformations are defined via differential equations and require iterative evaluations at inference time—often requires hundreds of neural function evaluations (NFEs) per sample. This makes generation computationally expensive and limits real-time applicability, motivating the development of high-quality single-step models, including consistency models \cite{song2023consistency,songimproved}, flow maps \cite{boffi2024flow,boffi2025build}, and Meanflows \cite{geng2025mean,huang2026rmflow}.

Recently, drifting models \cite{deng2026generatihve} have been proposed to address this inefficiency by shifting distribution matching entirely to training. In this framework, the neural network $f_\theta$ acts as a single-step generator, mapping a simple prior distribution $q$ (e.g., Gaussian noise) directly to samples in data space. At any training time $t$, the model induces a pushforward distribution $q_t = f_{\theta_t \sharp} q$. The goal of learning is to evolve $q_t$ so that it matches the target data distribution $p$.
Crucially, the model does not explicitly learn a vector field parameterized by a neural network. Instead, training is guided by a pre-designed drifting field $V_{p, q_t}(x)$, which depends on both the current generated distribution $q_t$ and the target distribution $p$. This field prescribes how a sample $x_t \sim q_t$ should move so as to reduce the discrepancy between $q_t$ and $p$. Intuitively, $V_{p,q_t}(x)$ can be viewed as a velocity that attracts generated samples toward regions of high data density while repelling them from over-represented regions of $q_t$.

During training, the neural network parameters evolve over iterations, $\theta_t \to \theta_{t+\Delta t}$, which induces a corresponding evolution of generated samples. For a fixed noise input $\epsilon \sim q$, the generated sample follows a trajectory $x_t = f_{\theta_t}(\epsilon)$. The parameter update thus results in a displacement
\[
x_{t+\Delta t} - x_t \approx V_{p,q_t}(x_t)\,\Delta t,
\]
meaning that the change in samples is driven by the drifting field. In this sense, the optimization of network parameters implicitly transports samples according to $V_{p,q_t}$.
From a continuous-time viewpoint \cite{turan2026generative}, taking the limit $\Delta t \to 0$ leads to the ordinary differential equation (ODE)
\begin{equation}\label{eq:drifting-ode-intro}
\frac{d x_t}{dt} = V_{p,q_t}(x_t).
\end{equation}
This ODE characterizes the training dynamics at the level of samples: the generated distribution $q_t$ evolves by flowing along the vector field $V_{p,q_t}$. At equilibrium, when $q_t = p$, the drifting field vanishes and the dynamics reach a fixed point. See Section~\ref{sec:spectral-viewpoint} for a more detailed derivation. 

The drifting field is constructed via kernel smoothing (cf.~Section~\ref{sec:background}), where interactions between samples are weighted by a similarity kernel (e.g., Gaussian), resulting in an attraction–repulsion mechanism that quantifies the discrepancy between the data distribution $p$ and the generated distribution $q$. However, this kernel-based smoothing inherently acts as a low-pass filter in the Fourier domain, attenuating high-frequency components of the distributions \cite{turan2026generative}. As a consequence, the resulting dynamics exhibit a pronounced spectral bias: while low-frequency modes converge rapidly, high-frequency modes decay much more slowly—exponentially so under a Gaussian kernel. A similar, albeit less severe, issue persists when using a Laplacian kernel. This spectral bias creates a fundamental bottleneck, slowing convergence and limiting the model’s ability to accurately capture fine-scale features of the data; see detailed analysis in Section~\ref{sec:spectral-viewpoint}. 

\subsection{Our Contributions}
Motivated by the spectral bottleneck described above, we propose \emph{second-order drifting models}, which lift the dynamics into phase space by introducing an auxiliary velocity variable $v$:
\begin{equation}\label{eq:second-order-drifting-model-intro}
\begin{cases} 
\frac{dx}{dt} = v, \\ 
\frac{dv}{dt} = V_{p,q_t}(x) - \gamma v,
\end{cases}
\end{equation}
where $\gamma > 0$ is a damping coefficient (can be $t$-dependent). By incorporating inertia/velocity, the resulting dynamics can effectively mitigate the low-pass bias induced by kernel smoothing. 

At the distribution level, the linearized dynamics follow second-order ODEs in Fourier space, turning each mode into an accelerated system analogous to the Nesterov \cite{nesterov1983method} dynamics. As a result, convergence is no longer constrained by the kernel spectrum $\lambda(\omega)$, but can be regulated via the damping schedule, effectively reducing spectral stiffness and accelerating high-frequency recovery. Guided by this insight, we develop a semi-implicit second-order training algorithm that preserves one-step generation while improving stability, convergence, and fine-scale fidelity. Empirically, our method achieves consistent improvements across synthetic generation, sequential data generation, and robotic control tasks. In summary, our contributions are threefold:
\begin{itemize}[leftmargin=5mm]
    \item \textbf{Acceleration perspective.} We view first-order drifting as spectrally imbalanced and introduce momentum to counter its low-pass bias.
    \item \textbf{Second-order dynamics.} We propose a phase-space formulation whose Fourier dynamics recover the Nesterov acceleration, reducing dependence on the kernel spectrum.
    \item \textbf{Algorithm and validation.} We develop a semi-implicit method that preserves one-step inference and improves convergence and high-frequency fidelity across tasks.
\end{itemize}

\subsection{Additional Related Works}
{\bf Acceleration and Second-Order Optimization Dynamics.}
Momentum-based acceleration has a long history in optimization (cf.~\cite{polyak1964some,nesterov1983method}). Continuous-time second-order ODE formulations of these methods have been extensively studied to characterize convergence rates and guide algorithm design (cf.~\cite{attouch2011continuous,su2016differential,attouch2018fast,attouch2016hessian,shi2019acceleration,wibisono2016variational,kovachki2021continuous}). Such perspectives have also inspired acceleration techniques for neural ODE-related models \cite{norcliffe2020second,nguyen2020momentumrnn,xia2021heavy,wang2022does,nguyen2022improving} and diffusion models~\cite{dockhorn2022cld}. In this paper, we bring this optimization viewpoint to drifting models: By linearizing the dynamics, each Fourier mode of the density perturbation behaves like a scalar optimization problem with conditioning governed by the kernel spectrum. This perspective enables the design of second-order 
dynamics 
to accelerate distribution matching.

\medskip
\noindent{\bf Advances in Drifting Models.} 
Several recent works have extended the drifting framework in complementary directions. We highlight a few representative advances here, noting that a comprehensive review is beyond the scope of this paper due to space limitations. The authors of \cite{turan2026generative} relate drifting model training to score matching, providing a theoretical foundation for our work. Sinkhorn drifting exploits connections to Sinkhorn gradient flows to improve performance in low-temperature regimes \cite{he2026sinkhorn}. A unified perspective relates drifting to score-based models by showing that the mean-shift direction corresponds to score mismatch, providing a coherent view across temperature regimes \cite{lai2026unified}. From a transport perspective, a long–short flow-map formulation derives drifting from flow-based models and achieves competitive performance with reduced batch sizes \cite{li2026long}. Gradient flow drifting interprets the method as a Wasserstein gradient flow of a kernel density estimation (KDE)-based divergence, enabling the design of alternative velocity fields that mitigate mode collapse \cite{cao2026gradient}. 
In addition, analytical correction methods address minibatch-induced bias in the drifting field through closed-form adjustments \cite{zhang2026analytical}, while friction-augmented variants introduce damping mechanisms to stabilize the dynamics and improve performance in domain translation tasks \cite{kazanskii2026attraction}.

\section{Background: Drifting Models}\label{sec:background}
In this section, we review the drifting model \cite{deng2026generatihve}. Let $p$ be a target distribution on $\mathbb{R}^d$, $f_\theta:\mathbb{R}^k \to \mathbb{R}^d$ a generator, and $q := (f_\theta)_\sharp \mathcal{N}(0,I)$ the induced model distribution, where $\gN(0,I)$ is standard Gaussian. The drifting model learns a kernel-based \emph{drift vector field} directly from samples.

For a positive kernel $k(x,y)$, a particular construction of the drift field \cite{deng2026generatihve} is given by
$$
V_{p,q}(x) = V_p(x) - V_q(x),
$$
with
$$
V_p(x) = \frac{\mathbb{E}_{y\sim p}[k(x,y)(y-x)]}{\mathbb{E}_{y\sim p}[k(x,y)] },\ 
V_q(x) = \frac{\mathbb{E}_{y\sim q}[k(x,y)(y-x)]}{\mathbb{E}_{y\sim q}[k(x,y)] }.
$$
The drifting field $V_{p,q}$ defines stop-gradient targets for generated samples, and $f_\theta$ (with initialization being $\theta_0$) is updated so that its pushforward distribution matches the drifted distribution following:
\begin{itemize}
    \item Sample $\epsilon\sim \mathcal{N}(0,I)$ and compute $x_t=f_{\theta_t}(\epsilon)$ 
    
    \item Drift $x_{t+\Delta t}=x_t+V_{p,q_t}(x_t)$

    \item Update $\theta_t$ by backpropagating the following loss function: 
    $$
    \gL(\theta)=\mathbb{E}_{\epsilon}[\|f_{\theta_t}(\epsilon)-\texttt{sg}[x_{t+\Delta t}] \|^2],
    $$
    where $\texttt{sg}$ denotes the stop-gradient operation.
\end{itemize}
After convergence, $f_\theta$ directly maps the prior distribution $\gN(0,I)$ to the data distribution $p$, enabling data generation with a single function evaluation (1-NFE).

\section{A Spectral Viewpoint of Training Drifting Models}
\label{sec:spectral-viewpoint}

In this section, we revisit the spectral analysis of drifting models \cite{turan2026generative}, demonstrating that first-order drifting causes frequency-dependent convergence and intrinsic spectral stiffness.

Starting from the ODE \eqref{eq:drifting-ode-intro}, and let $q_t(x)$ denote the density of the data, then we have \cite{turan2026generative}
\begin{equation}\label{eq:drift-pde}
\partial_tq_t(x) + \nabla\cdot \big( q_t(x)V_{p,q_t}(x) \big) = 0.
\end{equation}

\paragraph{Linearization.}
We consider linearizing \eqref{eq:drift-pde} for a general positive, radial symmetric kernel $k(x,y)=K(\|x-y\|)$ with $K$ being integrable and decaying in $\|x-y\|$; notice that both Gaussian and Laplacian kernels satisfy these conditions. To linearize \eqref{eq:drift-pde} around the equilibrium $p$, let $\rho_t:=q_t-p$.
Under mild assumptions (see Appendix~\ref{appendix-linearize}), we can show that $\rho_t$ satisfies the PDE:
\begin{equation}\label{eq:linearized-pde-residual}
\partial_t\rho_t(x) = -\nabla\cdot\Big(\int M(x-y)\rho_t(y)dy\Big) := - \nabla \cdot(M*\rho_t)(x),
\end{equation}
where $M(x)=cx K(x)$ with $c$ being a constant.

Taking the Fourier transform of \eqref{eq:linearized-pde-residual} yields a decoupled system:
\begin{equation}\label{eq:fourier-ode-rev}
\partial_t \hat{\rho}_t(\omega) = \lambda(\omega)\hat{\rho}_t(\omega), 
\quad 
\lambda(\omega) := \omega \cdot \nabla_\omega \hat{K}(\omega),
\end{equation}
where $\hat{K}(\omega)$ is the Fourier transform of $K$. Thus, each frequency evolves independently and the dynamics are fully characterized by the spectral rate $\lambda(\omega)$.

\begin{proposition}[Spectral stiffness of first-order drifting]
Under the above assumptions, the Fourier modes of the residual satisfy
\[
\partial_t \hat \rho_t(\omega) = \lambda(\omega)\hat \rho_t(\omega),
\quad \lambda(\omega)\le 0.
\]
Consequently, the time required to contract mode $\omega$ by a factor $\varepsilon$ is
\[
T_\varepsilon(\omega)=\frac{1}{|\lambda(\omega)|}\log\frac{1}{\varepsilon}.
\]
\end{proposition}
This proposition reveals a key limitation: convergence is controlled by the inverse spectral rate $|\lambda(\omega)|^{-1}$. Modes with small $|\lambda(\omega)|$ become bottlenecks, resulting in slow convergence.


\paragraph{Examples.}
We now examine the spectral rate $\lambda(\omega)$ for two standard kernels.

\begin{itemize}[leftmargin=5mm]
\item \textbf{Gaussian kernel.}  
For $K_G(z)=\exp(-\|z\|^2/(2\sigma^2))$, we obtain
\[
\lambda(\omega) = -\frac{\sigma^2}{(2\pi\sigma^2)^{d/2}} \|\omega\|^2 
\exp\Big(-\frac{\sigma^2}{2}\|\omega\|^2 \Big) < 0.
\]
In particular, $|\lambda(\omega)|$ decays exponentially in $\|\omega\|$, which implies
\[
T_\varepsilon(\omega) \propto \frac{\exp(\frac{\sigma^2}{2}\|\omega\|^2)}{\|\omega\|^2}\log\frac{1}{\varepsilon}.
\]
Thus, high-frequency modes converge exponentially slowly.

\item \textbf{Laplacian kernel.}  
For $K_L(z)=\exp(-\|z\|/\sigma)$, we obtain
\[
\lambda(\omega) = -\frac{C_d(d+1)\sigma^2\|\omega\|^2}{(1+\sigma^2\|\omega\|^2)^{(d+3)/2}} < 0,
\]
where $C_d$ is a dimension-dependent constant. This yields
\[
T_\varepsilon(\omega) \propto \frac{(1+\sigma^2\|\omega\|^2)^{(d+3)/2}}{\|\omega\|^2}\log\frac{1}{\varepsilon}.
\]
Here, the decay is polynomial, but still leads to severe spectral imbalance.
\end{itemize}

\paragraph{Interpretation (low-pass dynamics).}
Since $\hat{K}(\omega)$ decays with $\|\omega\|$, the spectral rate $\lambda(\omega)$ vanishes at high frequencies. As a result, the dynamics preferentially eliminate low-frequency errors, while high-frequency components persist for long times. In effect, the kernel induces a low-pass filter on the residual dynamics. This intrinsic spectral stiffness is a structural consequence of kernel smoothing, and explains why first-order drifting alone leads to slow convergence of fine-scale structure.


\section{Second-Order Drifting Models}
\label{sec:second-order-drifting-models}
From the previous analysis, the convergence of $\hat\rho_t(\omega)$ is frequency-dependent: for Gaussian kernels, it is exponentially slow at large $\omega$, while for Laplacian kernels it is slow at both high and low frequencies. For the Gaussian case, \cite{turan2026generative} mitigates this via an annealed bandwidth $\sigma(t)$. 
However, this approach is kernel-specific, raising the question:

\begin{center}
\emph{
Can we accelerate drifting model training for positive, radial, and translation-invariant kernels, ideally making iteration complexity frequency-independent?
}
\end{center}

To address this question, we reinterpret the linearized first-order dynamics from an optimization perspective. Since $\lambda(\omega)\leq 0$ (cf.~Proposition~\ref{prop:lambda-omega}), the Fourier residual evolves as
\[
    \hat{\rho}_t(\omega)=\hat{\rho}_0(\omega)\exp(\lambda(\omega)t).
\]
When $|\lambda(\omega)|$ is small, the corresponding mode decays slowly. Equivalently, for fixed $\omega$, \eqref{eq:fourier-ode-rev} can be viewed as the continuous-time limit of gradient descent on the convex quadratic objective
\[
    \min_{\hat{\rho}}-\frac{1}{2}(\lambda(\omega)) \hat{\rho}^2(\omega).
\]
This viewpoint motivates importing acceleration from optimization. Rather than evolving the density residual via a first-order ODE, we introduce second-order damped ODE. 

\subsection{Accelerated Second-Order Dynamics via Nesterov ODE}
To overcome the spectral stiffness of first-order drifting, we directly adopt a
time-dependent second-order dynamics inspired by Nesterov acceleration. In the
Fourier domain, we propose the evolution
\begin{equation}
\partial_t^2 \hat{\rho}_t(\omega) + \alpha_t \, \partial_t \hat{\rho}_t(\omega) - \lambda(\omega)\hat{\rho}_t(\omega) = 0,
\label{eq:nesterov_ode}
\end{equation}
where $\alpha_t = \frac{\alpha}{t}$ with $\alpha \geq 3$.

\Eqref{eq:nesterov_ode}
can be viewed as the continuous-time limit of Nesterov's
accelerated gradient method \cite{su2016differential}. Compared to the first-order ODE, which evolve each
frequency independently as $\partial_t \hat{\rho}_t(\omega) = \lambda(\omega)\hat{\rho}_t(\omega)$,
the second-order ODE introduces inertia, allowing faster decay of modes
with small $|\lambda(\omega)|$. In particular, classical results suggest that
such dynamics achieve accelerated convergence rates of order $\mathcal{O}(1/t^2)$
for convex objectives, significantly mitigating the dependence on the spectral
rate $\lambda(\omega)$.


\paragraph{Why constant damping is insufficient.} In optimization theory, another natural alternative is the constant-damping second-order system
\begin{equation}
\partial_t^2 \hat{\rho}_t(\omega) + \gamma \, \partial_t \hat{\rho}_t(\omega) - \lambda(\omega)\hat{\rho}_t(\omega) = 0,
\label{eq:heavy_ball}
\end{equation}
which corresponds to the heavy-ball method \cite{polyak1964some}. Its dynamical behavior depends on the roots of
the characteristic equation
\[
r^2 + \gamma r - \lambda(\omega) = 0.
\]
The dynamics exhibit three regimes depending on the relation between $\gamma^2$ and $-4\lambda(\omega)$:
\begin{itemize}[leftmargin=5mm]
    \item {\it Underdamped $\gamma^2<-4\lambda$:} the two roots are
    \(
        r=-\gamma/2\pm i\sqrt{|4\lambda+\gamma^2|}/2,
    \)
    and the solution is
    \[
    \hat{\rho}_t(\omega) = e^{-\frac{\gamma}{2}t}\big(C_1\cos(\sqrt{|4\lambda+\gamma^2|}/2t )+C_2\sin(\sqrt{|4\lambda+\gamma^2|}/2t) \big),\ C_1,C_2\ \text{are constants}.
    \]
    In this regime, the residual oscillates while its envelope decays at rate $e^{-\gamma t/2}$.

    \item {\it Critically-damped $\gamma^2=-4\lambda$:} the characteristic equation has one repeated real root $r=-\gamma/2$, and the solution is
    \[
    \hat{\rho}_t(\omega) = e^{-\frac{\gamma}{2}t}(C_1+C_2t),\ C_1,C_2\ \text{are constants}.
    \]
    This regime achieves the fastest non-oscillatory decay for a given mode.

    \item {\it Overdamped $\gamma^2>-4\lambda$:} the two roots are
    \(
        r_{1,2}=-\gamma/2\pm \sqrt{\gamma^2+4\lambda}/2,
    \)
    and the solution is
    \[
    \hat{\rho}_t(\omega) = C_1e^{r_1t} + C_2e^{r_2t},\ C_1,C_2\ \text{are constants}.
    \]
    In this case, the decay can become slow when the larger root is close to zero.
\end{itemize}
When $\gamma^2 \le -4\lambda(\omega)$, the mode $\hat{\rho}_t(\omega)$ decays at a rate governed by $\gamma$ rather than $|\lambda(\omega)|$, mitigating frequency dependence in the underdamped or critically damped regimes. However, no fixed $\gamma$ can critically damp all frequencies: modes with $\gamma^2 > -4\lambda(\omega)$ become overdamped and may converge more slowly. 
Therefore, constant damping cannot fully eliminate spectral imbalance. Designing more effective time-dependent damping schedules is an important yet challenging problem. Such designs could not only lead to new training algorithms for drifting models, but also provide new insights for acceleration in optimization. This remains a challenging direction for future research.

\subsection{Designing Second-Order Drifting Models}
In this subsection, we design a drifting model whose frequencies follow the second-order Nesterov ODE in the previous subsection.
We start from the following particle ODE governing the data evolution:
\begin{equation}\label{eq:nesterov-ode}
\begin{aligned}
\dot x^{(i)} &= v^{(i)},\\
\dot v^{(i)} &= V_{p,q_t}(x^{(i)}) - \frac{\alpha}{t}v^{(i)}.
\end{aligned}
\end{equation}
A stable semi-implicit discretization gives
\begin{equation}\label{eq:nesterov-particle}
\begin{aligned}
v^{(i)}_{n+1} &= \Big(1-\frac{\alpha}{n+1}\Big)v^{(i)}_n + \eta_1\cdot V_{p,q^n}(x^{(i)}_n),\\
x^{(i)}_{n+1} &= x^{(i)}_n + \eta_2\cdot v^{(i)}_{n+1},
\end{aligned}
\end{equation}
where $\eta_1,\eta_2>0$ are two constant. 

In training, the stop gradient is applied to $V_{p,q}$ exactly as in the original drifting model. In particular, let $x_N$ be the result of $N$ inertial drifting steps from $x_0=f_{\theta_t}(\epsilon)$, the loss is given by
\begin{equation}
\label{eq:nesterov-loss}
    \gL(\theta)=\mathbb{E}_{\epsilon}[\|f_{\theta_t}(\epsilon)-\texttt{sg}[x_{N}] \|^2],
\end{equation}



Next, we show that, under the second-order drifting dynamics in \eqref{eq:nesterov-particle}, each Fourier mode evolves according to the Nesterov ODE in \eqref{eq:nesterov_ode}.
From the particle dynamics \eqref{eq:nesterov-ode} and kinetic theory \cite{villani2002review}, we know the joint density of particle and velocity $Q_t(x,v)$ satisfies the PDE:
\begin{equation}\label{eq:phase-equation}
\partial_tQ_t + \nabla_x\cdot(vQ_t) + \nabla_v\cdot\Big(\Big[ V_{p,q_t}(x)-\frac{\alpha}{t}v \Big]Q_t \Big) = 0.
\end{equation}
Integrating \eqref{eq:phase-equation} over all $v$, we obtain the standard continuity equation (cf. Appendix~\ref{appendix-pde-second-order}):
\begin{equation}\label{eq:mass-conservation}
\partial_tq_t + \nabla_x\cdot(q_tu_t) = 0,
\end{equation}
where $q_t(x)=\int Q_t(x,v)dv$ and $u_t(x)=\frac{1}{q_t(x)}\int vQ_t(x,v)dv$.

Furthermore, multiply \eqref{eq:phase-equation} by $v$ and integrate over $v$, we arrive at the momentum equation (cf. Appendix~\ref{appendix-pde-second-order}):
\begin{equation}\label{eq:momentum-conservation}
\partial_t(q_tu_t) + \frac{\alpha}{t}(q_tu_t) = q_tV_{p,q_t}(x).
\end{equation}

Let $q_t = p+\rho_t$. Since $p$ is the equilibrium, its velocity $u$ and drift $V$ are zero. Thus, for small $\rho_t$, the terms $q_tu_t\approx pu_t$ and $q_tV_{p,q_t}\approx pV_{p,q_t}$. Moreover, we have already shown that the linearized drift field is
$V_{p,q_t}(x) \approx -(M*\rho_t)(x)$. 
Substituting these into \eqref{eq:mass-conservation} and \eqref{eq:momentum-conservation}, we have
$$
\begin{aligned}
\partial_t\rho_t + \nabla_x\cdot(p u_t) &= 0\\
\partial_t(pu_t) + \frac{\alpha}{t}(pu_t) &= -p(M*\rho_t).
\end{aligned}
$$
Take the time derivative of the linearized \eqref{eq:mass-conservation}:
$$
\partial_t^2\rho_t + \nabla_x\cdot\partial_t(pu_t) = 0.
$$
Substitute $\partial_t(pu_t)$ from \eqref{eq:momentum-conservation}:
$$
\partial_t^2\rho_t - \frac{\alpha}{t}\underbrace{ \nabla_x\cdot(pu_t)}_{-\partial_t\rho_t} - \nabla_x\cdot(p(M*\rho_t)) = 0.
$$
Using the substitution from the first derivative of $\rho_t$:
$$
\partial_t^2\rho_t + \frac{\alpha}{t}\partial_t\rho_t = \nabla\cdot(p(M*\rho_t)).
$$
Under the assumption of a locally flat prior $p$, 
we obtain the final form:
$$
\partial_t^2\rho_t(x) + \frac{\alpha}{t}\partial_t\rho_t(x) = \nabla\cdot(M*\rho_t)(x).
$$
Applying Fourier transform to the above equation, we have
$$
\partial_t^2\hat{\rho}_t(\omega) + \frac{\alpha}{t}\partial_t\hat{\rho}_t(\omega) = \lambda(\omega) \hat{\rho}_t(\omega).
$$

\section{Numerical Experiments}\label{sec:experiments}
In this section, we numerically verify the advantages of the proposed second-order models over other celebrated baseline models. All experiments were conducted on NVIDIA GeForce RTX 3090 GPUs. We use Laplacian kernel-based drift field for all the experiments.

\begin{wraptable}{r}{0.65\textwidth}
\centering
\begin{tabular}{lccc}
\toprule
& \multicolumn{2}{c}{Swiss roll} \\
\cmidrule{2-3}
Model 
& KL($\downarrow$)  & MMD ($\downarrow$)\\
\bottomrule\\[-0.5em]
Drifting 
& $.0127_{\pm .004}$ & $.000731_{\pm .00002}$ \\
Drifting $2^{\text{nd}}$ (\textbf{ours}) 
& $.0089_{\pm .003}$ &  $.000719_{\pm .00001}$\\
\bottomrule
\end{tabular}
\caption{ Comparison of KL divergence and MMD between the drifting and second-order drifting models on the 2D 
Swiss roll synthetic task. 
}
\label{tb:2d_kl_mmd}
\end{wraptable}

\subsection{Synthetic Tasks}
The Swiss roll is a smooth but highly folded manifold where high-frequency structure shows up through rapid changes in direction and curvature. Even though the density varies smoothly, the folding creates fine-grained geometric variation that requires tracking sharp changes in trajectory to follow the manifold correctly. This makes it a useful test for identifying whether a model can preserve high-frequency behavior. 
Additional experiment details are found in Appendix~\ref{appendix:experimental-synthetic}.

At each training iteration, the model generates $2048$ samples, from which we estimate the empirical density using histograms and compute the KL divergence to the data distribution. As shown in Figure~\ref{figure:swiss_roll_metrics} (a), our model exhibits better convergence in terms of KL divergence. At test time, we apply the same histogram-based density estimation procedure to the final models and report both KL divergence \cite{shlens2014notes} and maximum mean discrepancy (MMD) \cite{smola2006maximum} in Table~\ref{tb:2d_kl_mmd}, where our model consistently outperforms the baseline. Furthermore, we transform the estimated sample and data densities into the Fourier domain. Figure~\ref{figure:swiss_roll_metrics} (b) shows that our model achieves smaller Fourier coefficient errors, particularly at high-frequency modes.


\begin{figure}[!ht]
\centering
\begin{tabular}{c|cc}
\includegraphics[width=0.3\linewidth]{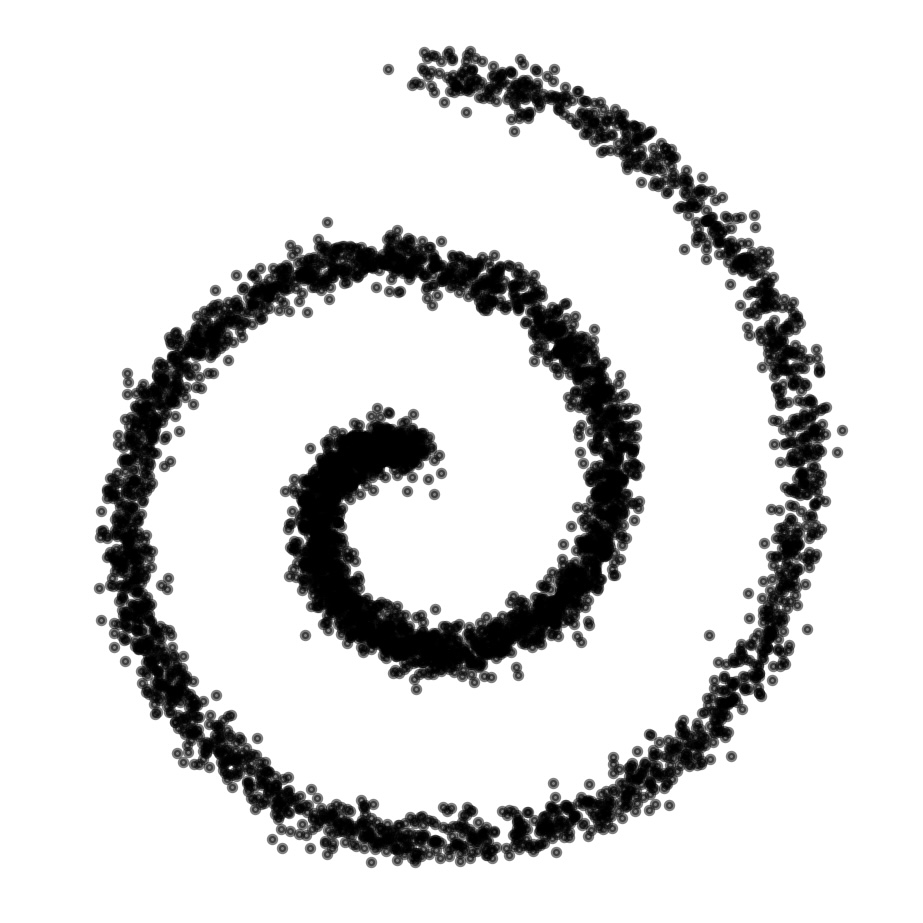}&  
\includegraphics[width=0.3\linewidth]{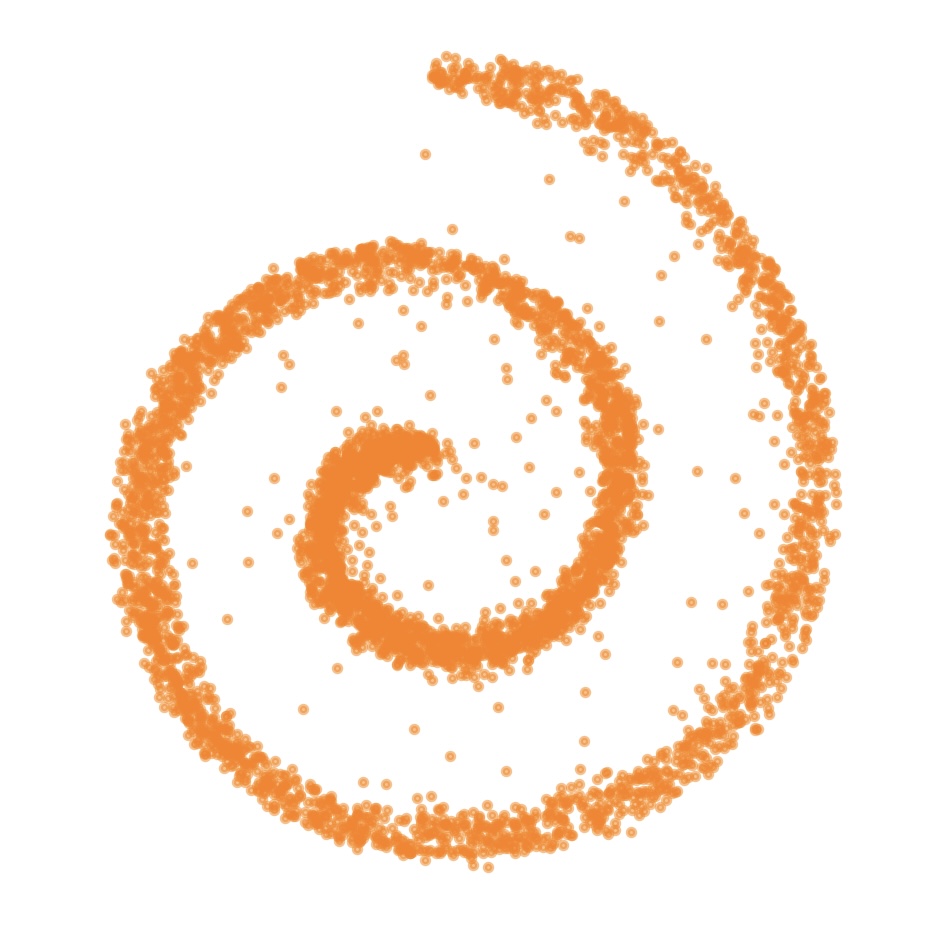}
&\includegraphics[width=0.3\linewidth]{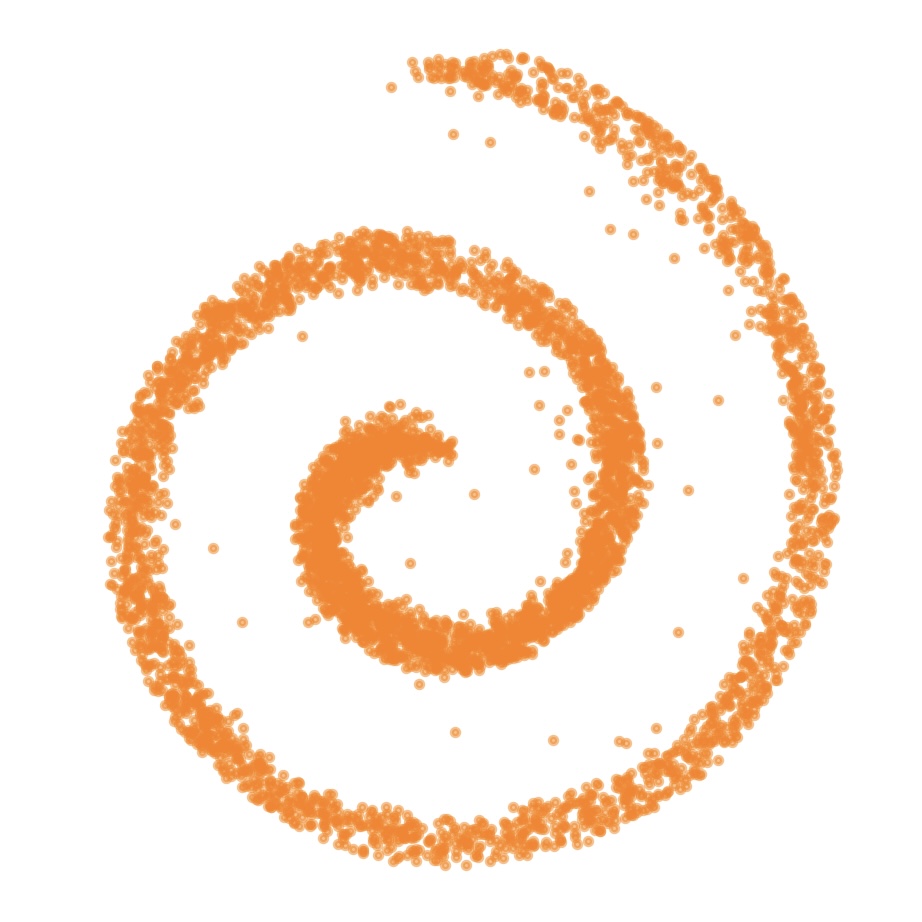}\\
{   (a)}& {   (b)} & {   (c)} 
\end{tabular}
\caption{ 
Swiss Roll experiments: (a) Ground truth (b) Drifting model samples. (v) Second-order drifting model samples. 
Our model improves frequency fidelity by reducing samples outside the \(\textit{Swiss\ roll}\) region.
}
\label{fig:swiss_roll}
\end{figure}

\begin{figure}[!ht]
\centering
\begin{tabular}{cc}
\includegraphics[width=0.45\linewidth]{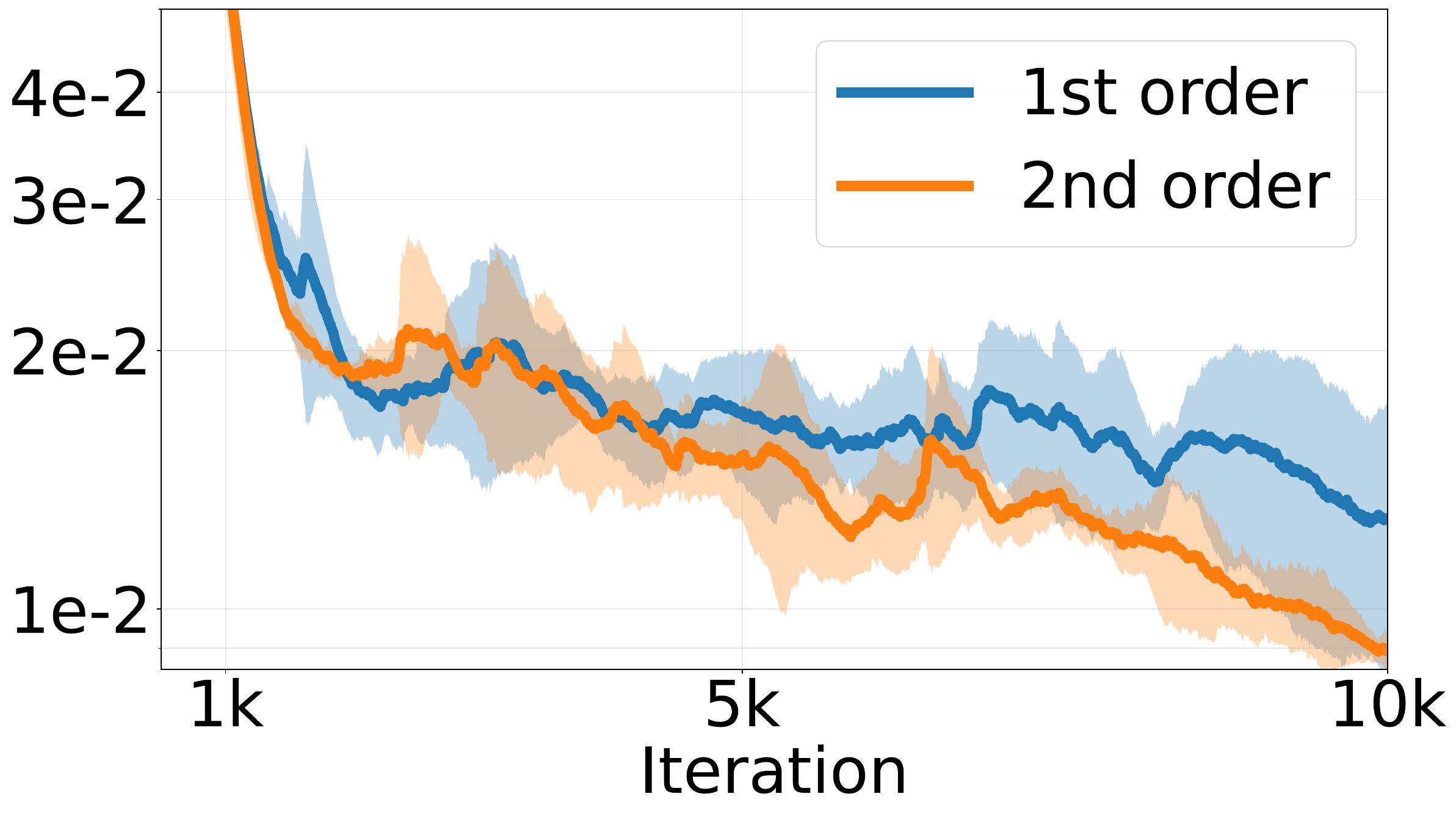}&  
\includegraphics[width=0.45\linewidth]{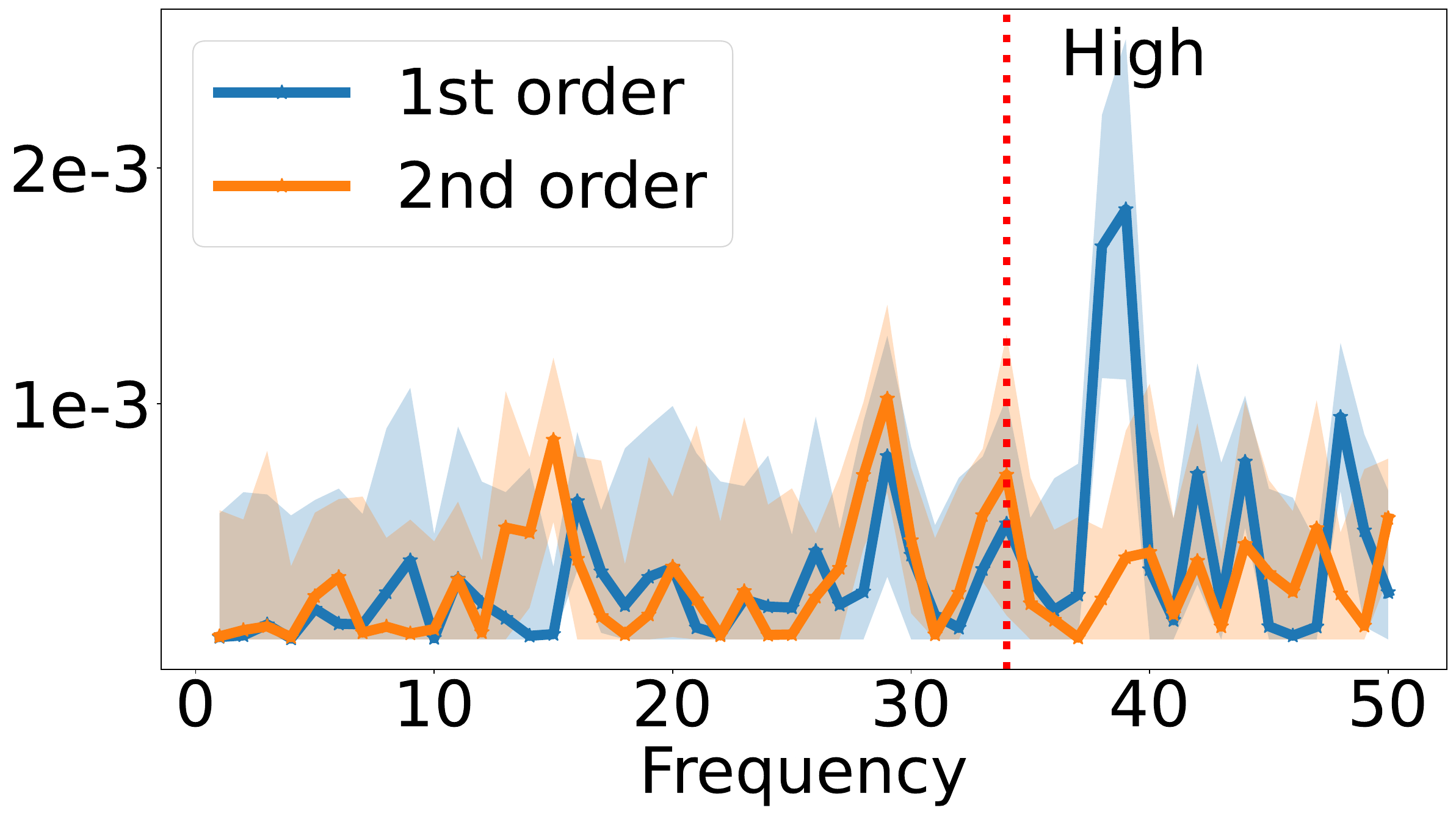} \\
{   (a)}& {   (b)}
\end{tabular}
\caption{ 
Swiss Roll experiments: (a) Running-average KL divergence vs. training iteration (5-run average). (b) Fourier coefficient error vs. frequency for the sampled distributions generated by the drifting model and our second-order drifting model, averaged over five runs. We estimate the empirical sample densities using histograms and compute the KL divergence and Fourier coefficient errors based on the resulting density estimates.
}
\label{figure:swiss_roll_metrics}
\end{figure}

\subsection{MNIST Image Generation}
\begin{wraptable}{r}{0.65\textwidth}
\centering
\begin{tabular}{lrrr}
\toprule
Model & FID($\downarrow$)  & NFE ($\downarrow$)\\
\midrule
Rectified Flow~\cite{guo2025variational} & $99.0_{\pm 3.0}$ & $2$ \\
Consistency Model\cite{guo2025variational} & $33.0_{\pm 5.0}$ & $2$ \\
VRFM\cite{guo2025variational} & $61.0_{\pm 7.0}$ & $2$ \\
\midrule
Drifting & $59.5_{\pm 3.0}$ & 1 \\
Drifting $2^{\text{nd}}$ (\textbf{ours}) & $\mathbf{48.2_{\pm 3.0}}$ & 1 \\
\bottomrule
\end{tabular}
\caption{
FID of different methods on MNIST dataset.
}
\label{tb:mnist_fid}
\end{wraptable}
We evaluate the drifting model and our second-order drifting model on MNIST image generation. MNIST consists of $28\times 28$ grayscale handwritten digit images and is a standard benchmark for evaluating generative models on low-resolution image distributions~\cite{lecun2002gradient}.

In this experiment, both the baseline drifting model and the proposed second-order drifting model are implemented with a Diffusion Transformer (DiT) backbone~\cite{peebles2023scalable}. At inference time, both models generate samples from Gaussian noise. We evaluate generation quality using the Fréchet Inception Distance (FID)~\cite{heusel2017gans}, where a lower FID indicates that the generated distribution is closer to the real MNIST distribution. The detailed experimental setup is in Appendix~\ref{appendix:experimental-mnist}.

\begin{figure}[!ht]
\centering
\begin{tabular}{cc}
\includegraphics[width=0.45\linewidth]{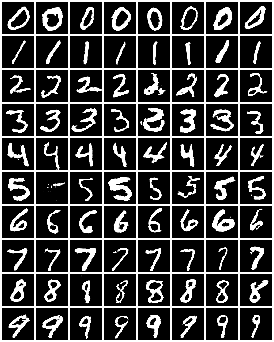}&  
\includegraphics[width=0.45\linewidth]{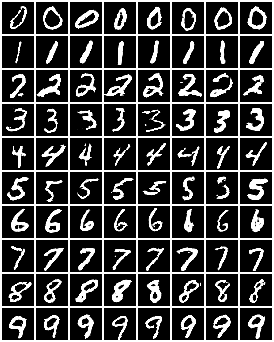}\\
\end{tabular}
\caption{
Samples generated by (left) drifting and (right) our second-order drifting models on MNIST dataset.
}\label{fig:mnist-example}
\end{figure}
Table~\ref{tb:mnist_fid} shows that the second-order drifting model consistently improves over the original drifting model in terms of FID. Figure~\ref{fig:mnist-example} shows some randomly generated samples by the baseline and our proposed second-order drifting models.
We further compare the drifting-based models with small-NFE flow-based generative models, including rectified flow~\cite{liu2023flow}, consistency models~\cite{song2023consistency}, and variational rectified flow (VRFM)~\cite{guo2025variational}.

\subsection{Sequential Data Generation: Dynamical Systems}
We further evaluate our proposed second-order drifting model on dynamical system trajectory generation tasks. Sampling trajectories in dynamical systems is a fundamental problem for forecasting and understanding complex time-dependent phenomena, such as chaotic dynamics, biological oscillations, and extreme events~\cite{perkins2013measurement,mosavi2018flood}. Recent works~\cite{finzi2023user,huang2026improving} have studied diffusion, flow-matching-based  and 1-NFE models for event-guided dynamical system trajectory sampling.

In this experiment, we formulate trajectory generation for the Lorenz~\cite{lorenz1963deterministic} and FitzHugh--Nagumo~\cite{fitzhugh1961impulses} systems as a time-series generation problem by discretizing the continuous time variable $t$ on a uniform grid, following the experimental setup in~\cite{finzi2023user,huang2026improving}. Each trajectory is represented as a sequence of states concatenated into
$\vx_{\rm data} = [\vx(\tau_m)]_{m=1}^M \in \mathbb{R}^{Md}$,
where $M$ denotes the total number of time steps, $d$ is the system dimension, and $\vx(\tau_m)\in\mathbb{R}^d$ is the system state at discretized time $\tau_m$. Given this representation, our goal is to learn a generative model that can produce realistic dynamical trajectories $\vx_{\rm data}$. For event-guided trajectory generation, where events are defined by a constraint set
$E=\{\vx_{\rm data}\mid C(\vx_{\rm data})>0\}$. Figure~\ref{tab:fig_events} illustrates the event structures for the Lorenz and FitzHugh--Nagumo systems. We condition the model on event guidance by using the corresponding class label as an additional input without relying on Tweedie's formula as in prior diffusion-based approaches~\cite{finzi2023user,huang2026improving}.

\begin{figure}[!ht]
\centering
\begin{tabular}{cc}
\includegraphics[width=0.45\linewidth]{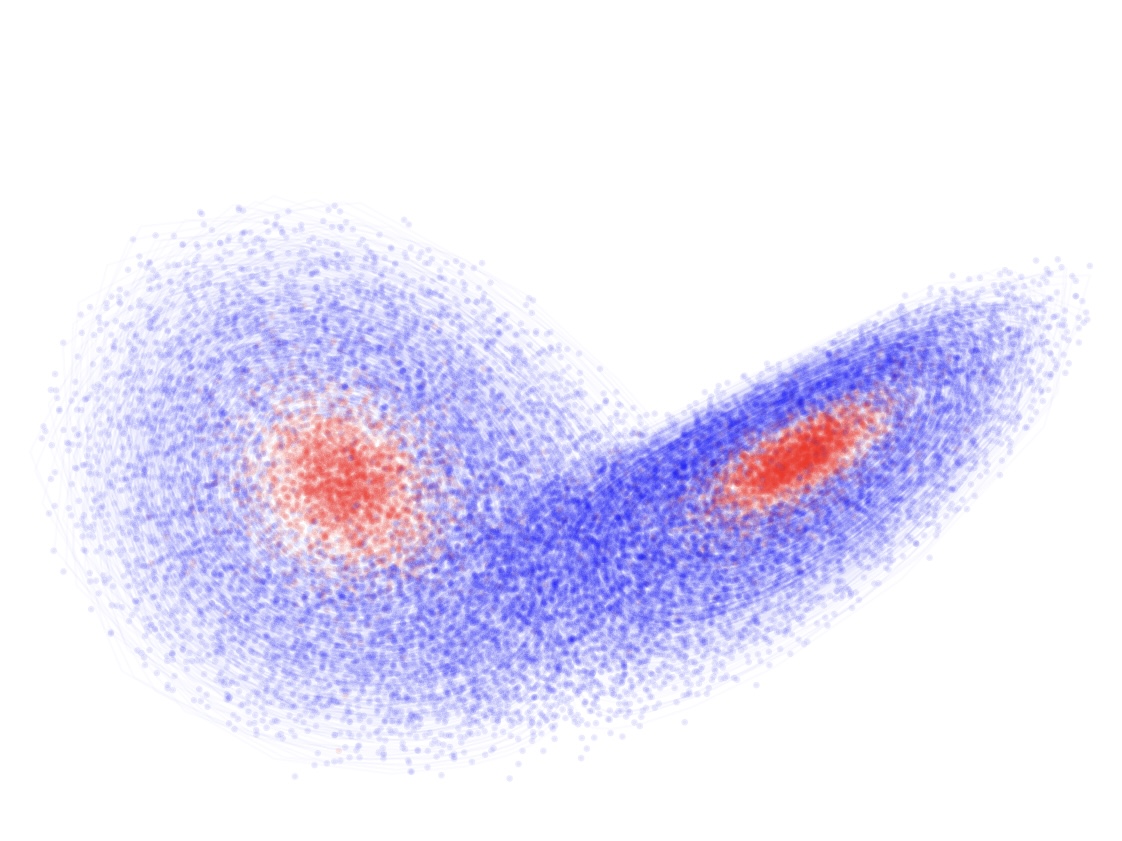}& \includegraphics[width=0.45\linewidth]{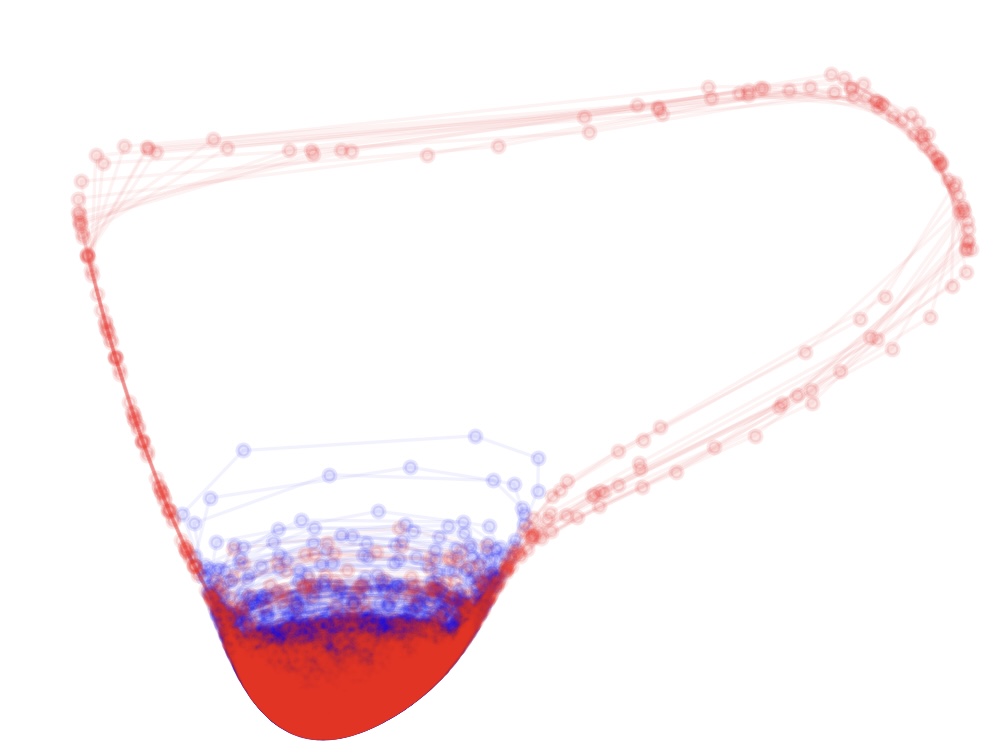}\\
{  (a)} & {   (b)}\\
\end{tabular}
\caption{ Events Illustration:
(a) Lorenz system trajectories; red paths remain in one attractor arm, corresponding to the condition $C(\vx)\geq 0$.
(b) FitzHugh–Nagumo trajectories; neuron spike events are shown in red, defined by $C(\vx)\geq 0$.
}
\label{tab:fig_events}
\end{figure}
We implement both the original drifting model and our second-order drifting model for this task Following prior work~\cite{finzi2023user,huang2026improving}, we use a U-Net backbone~\cite{finzi2023user} to model the trajectory distribution. Model performance is evaluated using the KL divergence between the histogram-based density estimates of the event-value distributions $C(\vx)$ for the sampled and dataset trajectories. The detailed experimental setup is in Appendix~\ref{appendix:experimental-dynamcis}.

Table~\ref{tb:dynamic_kl} shows that our second-order drifting model consistently improves trajectory generation quality over the original drifting model. In particular, the second-order update leads to lower distributional discrepancy under KL divergence, demonstrating its advantage in modeling complex dynamical system trajectories.



\begin{table}[!ht]
\centering
\begin{tabular}{lccccc}
\toprule
& \multicolumn{2}{c}{Lorenz} & \multicolumn{2}{c}{FitzHugh-Nagumo} \\
\cmidrule{2-5}
Model & w/o $E$ ($\downarrow$) &  w/ $E$ ($\downarrow$) &  w/o $E$ ($\downarrow$) &  w/ $E$ ($\downarrow$) & NFE ($\downarrow$)\\
\midrule
Diffusion~\cite{huang2026improving} & $0.0056$ & $0.2774$ & $0.0260$ & $0.3011$ & $128$ \\
FM~\cite{huang2026improving} & $0.0081$ & $0.2560$ & $0.0280$ & $0.3468$ & $128$ \\
\midrule
Drifting & $0.0101_{\pm .0004}$ & $0.3659_{\pm .002}$ & $0.0328_{\pm .0005}$ & $0.4115_{\pm .003}$ & 1 \\
Drifting $2^{\text{nd}}$ (\textbf{ours}) & $0.0089_{\pm .0003}$ & $0.3118_{\pm .003}$ & $0.0292_{\pm .0007}$ & $0.3788_{\pm .005}$ & 1 \\
\bottomrule
\end{tabular}
\caption{
KL divergence between the event $C(\vx)$ distributions of the generated trajectories and the dataset trajectories, estimated from histogram-based density approximations, with/without conditioning on the event.
}
\label{tb:dynamic_kl}
\end{table}

\subsection{Robotic Control}


Finally we consider the robotics tasks from~\cite{deng2026generatihve, chi2024diffusionpolicy}. Robotics manipulation is a sequential decision-making problem characterized by long-horizon dependencies, contact-rich interactions, and multi-modal action distributions. In such settings, control policies must maintain temporal consistency while adapting rapidly to changing object dynamics and task constraints.

We evaluate on a diverse suite of robotic manipulation benchmarks including Can, Tool Hang, Lift, and PushT, as well as multi-stage tasks such as BlockPush and Kitchen following the evaluation protocol of Drifting Policy~\cite{deng2026generatihve, chi2024diffusionpolicy}. 
Following a similar experimental setup as that in \cite{deng2026generatihve}, we replace the diffusion-based action generator with our second-order drifting model while retaining the same state-action representation and training pipeline. 
The detailed experimental setup is in Appendix~\ref{appendix:robotics}.
Table~\ref{table:robotics-results} shows that we achieve comparable or higher success rates across tasks while requiring substantially fewer training epochs. We compare our results with previous state-of-the-art models namely Diffusion Policy \cite{chi2024diffusionpolicy} and Drifting policy \cite{deng2026generatihve}. Specifically, second-order updates lead to significantly increased training efficiency on robotics tasks. 

We further study different initializations for the initial velocity $v^{(i)}_0$, including learned initialization, zero initialization, and historical warm-starting. We observe that zero initialization ($v_0^{(i)}=0$) provides the best trade-off between performance and efficiency, while historical warm-starting leads to instability due to non-stationary optimization dynamics. Learned initializations achieved the best success rates, but doubled the training time. Further details are provided in Appendix~\ref{appendix:robot-ablations}.

\begin{table}[!ht]
\centering
\begin{tabular}{l l c c c c c c}
\toprule
& & \multicolumn{2}{c}{Diffusion Policy} 
& \multicolumn{2}{c}{Drifting} 
& \multicolumn{2}{c}{Drifting 2nd (\textbf{ours})} \\
\cmidrule(lr){3-4} \cmidrule(lr){5-6} \cmidrule(lr){7-8}
Task & Setting 
& Success ($\uparrow$) & Epochs ($\downarrow$)
& Success ($\uparrow$) & Epochs ($\downarrow$)
& Success ($\uparrow$) & Epochs ($\downarrow$) \\
\midrule

\multirow{2}{*}{Can}       
& Visual & .97 & 3050 & \textbf{.99} & 50 & $\mathbf{.99_{\pm 0.016}}$ & \textbf{16} \\
& State  & .96 & 5000 & .98 & 50 & $\mathbf{.992_{\pm 0.01}}$ & \textbf{20} \\
\cline{1-8}

\multirow{2}{*}{Tool Hang} 
& Visual & .77 & 3050 & .67 & 25 & $\mathbf{.79_{\pm .051}}$ & \textbf{16} \\
& State  & .30 & 5000 & .38 & 50 & $\mathbf{.71_{\pm .06}}$ & \textbf{30} \\
\cline{1-8}

\multirow{2}{*}{Lift}      
& Visual & \textbf{1.0} & 3050 & \textbf{1.0} & 50 & $\mathbf{1.0_{\pm 0.0}}$ & \textbf{8} \\
& State  & .98 & 5000 & \textbf{1.0} & 50 & $\mathbf{1.0_{\pm 0.0}}$ & \textbf{25} \\
\cline{1-8}

\multirow{2}{*}{PushT}     
& Visual & .84 & 3050 & .86 & \textbf{100} & $\textbf{.87}_{\pm .018}$ & \textbf{100} \\
& State  & \textbf{.91} & 5000 & .86 & 50 & $.872_{\pm .023}$ & \textbf{30} \\
\cline{1-8}

\multirow{2}{*}{BlockPush} 
& Phase 1 & 0.36 & 5000 & \textbf{0.56} & 5000 & $.372_{\pm .061}$ & \textbf{3500} \\
& Phase 2 & 0.11 & 5000 & 0.16 & 5000 & $\mathbf{.168_{\pm .041}}$ & \textbf{3500} \\
\cline{1-8}

\multirow{4}{*}{Kitchen}   
& Phase 1 & \textbf{1.00} & 5000 & \textbf{1.00} & 300 & $\mathbf{1.00_{\pm 0.0}}$ & \textbf{100} \\
& Phase 2 & \textbf{1.00} & 5000 & \textbf{1.00} & 300 & $\mathbf{1.00_{\pm 0.0}}$ & \textbf{100} \\
& Phase 3 & \textbf{1.00} & 5000 & 0.99 & 300 & $\mathbf{1.00_{\pm 0.0}}$ & \textbf{100} \\
& Phase 4 & \textbf{0.99} & 5000 & 0.96 & 300 & ${.97_{\pm .015}}$ & \textbf{100} \\
\bottomrule
\end{tabular}
\caption{ Success rates (↑) and training epochs (↓) across tasks for diffusion policy and drifting variants.}
\label{table:robotics-results}
\end{table}


\section{Conclusion}\label{sec:conclusion}
We proposed second-order drifting models, a momentum-based extension of drifting that mitigates the spectral bias of kernel-driven dynamics. By lifting the evolution into phase space, we showed that the linearized dynamics follow accelerated second-order systems, reducing the dependence of convergence rates on the kernel spectrum. This provides a principled acceleration mechanism consistent with the interpretation of drifting as score-based gradient flow. We further introduced a semi-implicit training scheme that preserves one-step inference while improving convergence and stability. Experiments across multiple domains demonstrate consistent gains over first-order drifting. Overall, second-order drifting offers a simple and effective approach to accelerating one-step generative modeling. 
As a limitation, the current formulation is primarily designed to accelerate the training dynamics of drifting models, and its applicability to broader classes of generative frameworks remains an important direction for future work.


\appendix

\section{Linearized PDE}\label{appendix-linearize}
To derive the linearized PDE for $\rho_t = q_t-p$, we start from the continuous-limit evolution of the drifting model and linearize it around the equilibrium distribution $p$. The evolution of the model density $q_t(x)$ is governed by: 
\begin{equation}\label{eq:appendix-1}
\partial_tq_t(x) + \nabla\cdot\big( q_t(x)V_{p,q_t}(x) \big) = 0,
\end{equation}
where the drift field is defined as:
$$
V_{p,q_t} = V_p(x) - V_{q_t}(x),
$$
with $V_p(x) = \frac{\mathbb{E}_{y\sim p}[k(x,y)(y-x)] }{\mathbb{E}_{y\sim p}[k(x,y)] }$ and $V_{q_t}$ is defined similarly. 

Substituting $q_t=p+\rho_t$ into \eqref{eq:appendix-1}:
\begin{itemize}
    \item Time derivative: $\partial_t(p+\rho_t) = \partial_t \rho_t$ since $p$ is stationary.

    \item Drift at equilibrium: at equilibrium ($q_t=p$), the drift field vanishes, i.e., $V_{p,p}(x)=V_p(x)-V_p(x)=0$. 

    \item Expansion of $V_{q_t}$: we expand $V_{q_t}(x)$ to the first order in $\rho_t$. Let $D_\rho(x)=\int k(x,y)p(y)dy$ be the normalization factor. Then
    $$
    V_{p+\rho_t}(x) \approx V_p(x) + \frac{1}{D_p(x)}\int k(x,y)[(y-x)-V_p(x)]\rho_t(y) dy.
    $$
    Thus, the relative drift field $V_{p,q_t}$ becomes:
    $$
    V_{p,p+\rho_t}(x) = V_p(x) - V_{p+\rho_t}(x) \approx -\frac{1}{D_p(x)}\int k(x,y)[(y-x)-V_p(x) ]\rho_t(y)dy.
    $$
\end{itemize}

Substitute the above linearized drift back into the continuity equation and keep only terms linear in $\rho_t$:
\begin{equation}\label{eq:appendix-2}
\begin{aligned}
\partial_t\rho_t(x) &+ \nabla\cdot\big( p(x)\cdot V_{p,p+\rho_t}(x) \big) \approx 0,\\
\partial_t\rho_t(x) &= \nabla\cdot\Big( \frac{p(x) }{D_p(x)}\int k(x,y)[(y-x)-V_p(x) ]\rho_t(y)dy \Big).
\end{aligned}
\end{equation}

Furthermore, we make the following assumptions:
\begin{itemize}
    \item Radial symmetry kernel: $k(x,y)=K(\|x-y\|)$.

    \item We assume $p(x)$ is approximately constant relative to the scale of the kernel. This implies: (1) $V_p(x)\approx 0$---no local bias in the target distribution. (2) $\frac{p(x)}{D_p(x)}\approx c$, where $c$ is a constant, typically $\frac{1}{\int K(\|z\|)dz}$.
\end{itemize}
Let $z=x-y$, then \eqref{eq:appendix-2} simplifies to
$$
\partial_t\rho_t(x) = \nabla\cdot \Big( c\int K(\|x-y\|)(x-y)\rho_t(y)dy\Big) = \nabla\cdot(M*\rho_t)(x).
$$

\begin{remark}
We remark on the assumption that $p(x)$ is approximately constant relative to the kernel scale. Notice that the kernel $k(x,y)$ usually has a bandwidth $\sigma$ (e.g., in a Gaussian kernel). We assume the target distribution $p(x)$ is smooth and slowly varying relative to the width of the kernel. 
\end{remark}

\section{
Spectral Estimation
}
\begin{proposition}\label{prop:lambda-omega}
Let $k(x,y)=K(\|x-y\|)$ be a positive, radial symmetric kerne with $K$ being integrable and decaying in $\|x-y\|$, then 
$$\lambda(\omega)=\omega\cdot \nabla_{\omega}\hat{K}(\omega)\leq0,$$ 
where $\hat{K}(\omega)$ is the Fourier transform of $K$.
\end{proposition}
\begin{proof}
We will show that $\lambda(\omega)\leq 0$ for all $\omega$, and $\lambda(\omega)=0$ if and only if $\omega=0$. 


By definition, we have
$$
\hat{K}(\omega) = \int_{\mathbb{R}^d}K(z)e^{-i\omega\cdot z}dz,
$$
which is real and even since $K$ is real and even. Moreover, taking absolute values for the above equation and by triangle inequality, we have
$$
|\hat{K}(\omega)| \leq \int K(z)|e^{-\omega\cdot z}|dz = \int K(z) dz = \hat{K}(0).
$$
That is, $\omega=0$ is a global maximum. Moreover, $\hat{K}(\omega)=\phi(|\omega|)$ for some scalar function $\phi(r)$ with $\phi'(r)\leq 0$ for $r>0$. Therefore, we have 
$$
\omega\cdot \nabla_{\omega}\hat{K}(\omega) = \omega\cdot\Big( \phi'(|\omega|)\frac{\omega}{|\omega|}\Big) = |\omega|\phi'(|\omega|) \leq 0,
$$
with equality only at $\omega=0$.

\end{proof}

\section{PDE Limit for the Second-Order Drifting Models}\label{appendix-pde-second-order}
\subsection{Mass Conservation}
We define the marginal density $q_t(x)=\int Q_t(x,v)dv$ and the local mean velocity $u_t(x)=\frac{1}{q_t(x)}\int vQ_t(x,v)dv$. Integrating the phase space equation over all $v$:
\begin{itemize}
    \item 1. $\int \partial_tQ_tdv = \partial_tq_t(x)$

    \item 2. $\int \nabla_x\cdot(vQ_t)dv = \nabla_x\cdot(q_t(x)u_t(x))$

    \item 3. The $\nabla_v$ term vanishes by the divergence theorem (assuming $Q_t\to 0$ as $|v|\to\infty$)
\end{itemize}
This gives the standard continuity equation:
\begin{equation}\label{eq:mass-conservation-2}
\partial_tq_t + \nabla_x\cdot(q_tu_t) = 0.
\end{equation}

\subsection{Momentum Balance}
Multiply the kinetic equation by $v$ and integrate over $v$:
\begin{itemize}
    \item $\int v\partial_tQ_tdv = \partial_t(q_tu_t)$

    \item $\int v[v\cdot \nabla_xQ_t]dv = \nabla_x\cdot \int (v\otimes v)Q_tdv\approx 0$ (assuming low temperature/pressure where the velocity is negligible relative to the mean flow) \cite{carrillo2021quantifying}.

    \item $\int v[\nabla_v\cdot(VQ_t)]dv = -V_{p,q_t}(x)q_t(x)$ (integration by parts).

    \item $\int v[\nabla_v\cdot(-\frac{\alpha}{t}vQ_t)]dv = \frac{\alpha}{t}q_tu_t$ (integration by parts)
\end{itemize}
This gives the momentum equation:
\begin{equation}\label{eq:momentum-conservation-2}
\partial_t(q_tu_t) + \frac{\alpha}{t}(q_tu_t) = q_tV_{p,q_t}(x).
\end{equation}

\section{Experimental Setup}\label{appendix:experimental-setup}
\subsection{Implementation Details for 2D Synthetic}\label{appendix:experimental-synthetic}

\subsubsection{Hyperparameters of Drifting Models}
Table~\ref{tab:hyper_param_swissroll} shows the hyperparameters used for the second-order drifting model on the 2D Swiss roll trajectory sampling task. Our model utilizes a Nesterov ODE-inspired second-order formulation.

\begin{table}[!ht]
\centering
\begin{tabular}{lccccc}
\toprule
Task & $\eta_1$ & $\eta_2$ & $N$ & $\alpha$ & Temperature \\
\midrule
Swiss Roll & 0.25 & 0.25 & 4 & 3 & 0.05 \\
\bottomrule
\end{tabular}
\caption{ Hyperparameters of the second-order drifting model on the 2D Swiss roll task.}
\label{tab:hyper_param_swissroll}
\end{table}

\subsubsection{Dataset and Training Setup}
The 2D Swiss roll dataset consists of 2048 samples with additive Gaussian noise of standard deviation $0.03$. The task is purely two-dimensional and no normalization or scaling is applied to the data. The Swiss roll distribution contains highly folded local structure and sharp changes in density along the manifold, making it useful for evaluating a model's ability to capture fine-scale geometric structure during trajectory sampling.

Models are trained for 10{,}000 epochs using AdamW with a linearly annealed learning rate schedule. We do not use exponential moving average (EMA) during training.

\subsubsection{Model Architecture}
We use a residual MLP backbone for all 2D synthetic experiments. The network operates on a 32-dimensional latent representation with hidden dimension 256 and consists of six residual blocks. Each block applies LayerNorm, followed by a SiLU activation and a linear projection within a residual connection. The final output is scaled to remain within the range $[-4,4]$.

\begin{figure}[!ht]
\centering

\begin{minipage}[t]{0.42\linewidth}
\centering
\renewcommand{\arraystretch}{1.25}
\begin{tabular}{@{}l@{}}
\hline
\textbf{ResidualMLPBlock($d$)} \\
\hline
LayerNorm($d$) \\
Linear($d \rightarrow d$) \\
SiLU \\
Linear($d \rightarrow d$) \\
ResidualConnection \\
\hline
\end{tabular}
\end{minipage}
\hfill
\begin{minipage}[t]{0.50\linewidth}
\centering
\renewcommand{\arraystretch}{1.25}
\begin{tabular}{@{}lr@{}}
\hline
\multicolumn{2}{c}{\textbf{Residual MLP Architecture}} \\
\hline
Input Projection($2 \rightarrow 32$) & \\
Hidden Projection($32 \rightarrow 256$) & \\
\hline
ResidualMLPBlock($256$) & $\times 6$ \\
\hline
Output Projection($256 \rightarrow 2$) & \\
Output Scaling($[-4,4]$) & \\
\hline
\end{tabular}
\end{minipage}

\caption{ Residual MLP backbone architecture used for the 2D Swiss roll experiments.}
\label{fig:swissroll_architecture}
\end{figure}

\subsection{Experiment Details of the MNIST Task}\label{appendix:experimental-mnist}

\subsubsection{Hyperparameters of Drifting Models}
Table~\ref{tab:hyer_param_minist} shows the details of our second-order drifting model for dynamical system trajectory sampling tasks. Our model utilizes a Nesterov ODE-inspired second-order formulation.
\begin{table}[!ht]
\centering
\begin{tabular}{lccccc}
\toprule
Task &  $\eta_1$ & $\eta_2$ & $N$ & $\alpha$ & Temperature \\
\midrule
MNIST & 0.25 & 0.25 & 4 & 3 & [0.02, 0.05, 0.2] \\
\bottomrule
\end{tabular}
\caption{ Hyperparameters of the Second-order drifting model on MNIST tasks.}
\label{tab:hyer_param_minist}
\end{table}

\subsubsection{Model and Training Setup}
\textbf{Model Architecture}: We use diffusion transformer as backbone to learn the models. Figure~\ref{fig:driftdit_tiny_mnist_architecture} shows the detailed model architecture.

\textbf{Training Setup}. We train both models with a batch size of $256$. All models are trained for $200$ with learning rate $2\times 10^{-4}$. To evaluate performance, we sample $10000$ images from each model and $10000$ images from the real dataset, then compute the FID score.

\textbf{Resource Usage and Time}. We run all experiments on a single RTX 3090 GPU. Each training run requires approximately 8 GiB of GPU memory and takes around 10 hours.

\begin{figure}[!ht]
\centering
\begin{minipage}[t]{0.45\linewidth}
\centering
\renewcommand{\arraystretch}{1.25}
\begin{tabular}{@{}l@{}}
\hline
\textbf{DiTBlock($d$):} \\
\hline
RMSNorm($d$) \\
MultiHeadAttention(heads=$4$, QK-Norm, RoPE) \\
adaLN-Zero modulation \\
ResidualConnection \\
RMSNorm($d$) \\
SwiGLU(mlp\_ratio=$4$) \\
adaLN-Zero modulation \\
ResidualConnection \\
\hline
\end{tabular}
\end{minipage}
\hfill
\begin{minipage}[t]{0.50\linewidth}
\centering
\renewcommand{\arraystretch}{1.25}
\begin{tabular}{@{}lr@{}}
\hline
\multicolumn{2}{c}{\textbf{DriftDiT-Tiny (9M) Architecture:}} \\
\hline
PatchEmbed(img=$32$, patch=$4$, in=$1$, $d=256$) & \\
RegisterTokens($8$) & \\
LabelEmbed($10 \to 256$) & \\
AlphaEmbed(Fourier$\to$MLP) & \\
StyleEmbed(tokens=$32$, codebook=$64$) & \\
\hline
DiTBlock($d$) & $\times 6$ \\
\hline
FinalLayer($d$, patch=$4$, out=$1$) & \\
Unpatchify($32 \times 32 \times 1$) & \\
\hline
\end{tabular}
\end{minipage}

\caption{ DiT backbone architecture and the corresponding residual block for MNIST tasks.}
\label{fig:driftdit_tiny_mnist_architecture}
\end{figure}

\subsection{Implementation Details for Dynamical Systems}\label{appendix:experimental-dynamcis}
This section presents the implementation details for Lorenz and FitzHugh--Nagumo trajectory sampling, including the second-order drifting hyperparameters, the shared U-Net backbone and the training settings.

\subsubsection{Hyperparameters of Drifting Models}
Table~\ref{tab:hyer_param_dym} shows the details of our second-order drifting model for dynamical system trajectory sampling tasks. Our model utilizes a Nesterov-inspired second-order formulation.
\begin{table}[!ht]
\centering
\begin{tabular}{lccccc}
\toprule
Task &  $\eta_1$ & $\eta_2$ & $N$ & $\alpha$ & Temperature \\
\midrule
Lorenz & 0.25 & 0.125 & 16 & 3 & [0.02, 0.05, 0.2] \\
FitzHugh-Nagumo & 0.25 & 0.125 & 8 & 3 & [0.02, 0.05, 0.2] \\
\bottomrule
\end{tabular}
\caption{ Hyperparameters of the Second-order drifting model on Dynamical system tasks.}
\label{tab:hyer_param_dym}
\end{table}

\subsubsection{Model and Training Setup}
\textbf{Model Architecture}. We use the same U-Net backbone as described in Section~H of~\cite{finzi2023user} to parameterize both the drifting model and our second-order drifting model. Figure~\ref{fig:dym_unet_architecture} shows the architecture.
\begin{figure}[t]
\centering

\begin{minipage}[t]{0.42\linewidth}
\centering
\renewcommand{\arraystretch}{1.15}
\begin{tabular}{l}
\textbf{ResBlock($c$):} \\
\hline
GroupNorm(groups=$c/4$) \\
Swish \\
Conv(channels=$3$, ksize=$3$) \\
GroupNorm(groups=$c/4$) \\
Swish \\
Conv(channels=$3$, ksize=$3$) \\
SkipConnection \\
\hline
\end{tabular}
\end{minipage}
\hfill
\begin{minipage}[t]{0.54\linewidth}
\centering
\renewcommand{\arraystretch}{1.15}
\begin{tabular}{lc}
\hline
\multicolumn{2}{c}{\textbf{Convolutional U-Net Architecture:}} \\
\hline
ResBlock($c$)       & $\times 4$ \\
Downsample($2$)     &             \\
ResBlock($2c$)      & $\times 8$ \\
Downsample($2$)     &             \\
ResBlock($4c$)      & $\times 8$ \\
\hline
SkipResBlock($4c$)  & $\times 8$ \\
Upsample($2$)       &             \\
SkipResBlock($2c$)  & $\times 8$ \\
Upsample($2$)       &             \\
SkipResBlock($c$)   & $\times 4$ \\
\hline
Conv($128$)         &             \\
Conv($d$)           &             \\
\hline
\end{tabular}
\end{minipage}

\caption{U-Net backbone architecture and the corresponding residual block for Dynamical system tasks.}
\label{fig:dym_unet_architecture}
\end{figure}

\textbf{Training Setup}. For each task, we train both models with a batch size of $256$ using $8000$ trajectories generated by \texttt{torchdiffeq}, following~\cite{finzi2023user}. All models are trained for $3000$ epochs using a linearly decayed learning rate initialized at $2\times 10^{-4}$. To evaluate performance, we sample $2000$ trajectories from each model, estimate the distribution of the constraint value $C(\vx)$ using histograms, and compute the KL divergence between the generated and dataset distributions.

\textbf{Resource Usage and Time}. We run all experiments on a single RTX 3090 GPU. Each training run requires approximately 12 GiB of GPU memory and takes around 20 hours.

\subsection{Implementation Details for Robotics}
\label{appendix:robotics}
\subsubsection{Hyperparameters for Robotics experiments}
We summarize the configuration of our second-order drifting model across all benchmark tasks. The model is based on a Nesterov-inspired second-order update with semi-implicit discretization.  Full per-task configurations are provided in~\cref{appendix:robotics-hyperparameters}.
\begin{table}[!ht]
\centering
\begin{tabular}{l l c c c c c c}
\toprule
Task & Setting & $\eta_1$ & $\eta_2$ & $N$ & $\alpha$ & Temperature\\
\midrule
\multirow{2}{*}{Can}       
& Visual & .25 & .25 & 4 & 6.0 & [0.02, 0.05, 0.2]\\
& State  & .25 & .25 & 4 & 6.0& [0.02, 0.05, 0.2] \\ \cline{1-7}

\multirow{2}{*}{Tool Hang} 
& Visual & 0.0625 & 0.0625 & 16 & 1.0 & [0.02, 0.05, 0.2] \\
& State  & 0.0625 & 0.0625 & 16 & 1.0 & [0.02, 0.05, 0.2] \\ \cline{1-7}

\multirow{2}{*}{Lift}      
& Visual & 0.0625 & 0.0625 & 16 & 1.0 & [0.02, 0.05, 0.2] \\
& State  & 0.0625 & 0.0625 & 16 & 1.0 & [0.02, 0.05, 0.2] \\ \cline{1-7}

\multirow{2}{*}{PushT}     
& Visual & 0.0625 & 0.0625 & 16 & 1.0 & [0.02, 0.05, 0.2] \\
& State  & 0.0625 & 0.0625 & 16 & 1.0 & [0.02, 0.05, 0.2] \\

\midrule
\multirow{2}{*}{BlockPush} 
& Phase 1  & .25 & .25 & 4 & 1.0 & [0.02, 0.05, 0.2] \\
& Phase 2  & .25 & .25 & 4 & 1.0 & [0.02, 0.05, 0.2] \\

\midrule
\multirow{4}{*}{Kitchen}   
& Phase 1  & 0.125 & 0.125 & 4 & 1.0 & [0.02, 0.05, 0.2] \\
& Phase 2  & 0.125 & 0.125 & 4 & 1.0 & [0.02, 0.05, 0.2]\\
& Phase 3  & 0.125 & 0.125 & 4 & 1.0 & [0.02, 0.05, 0.2] \\
& Phase 4  & 0.125 & 0.125 & 4 & 1.0 & [0.02, 0.05, 0.2] \\
\bottomrule
\end{tabular}
\caption{ Hyperparameters used in robotics experiments}

\label{appendix:robotics-hyperparameters}
\end{table}



\begin{table}[!ht]
\centering
\begin{tabular}{l c}
\hline
Task / Modality & Downsample Channels \\
\hline
blockpush         & [256, 512, 1024] \\
can image         & [512, 1024, 2048] \\
can lowdim        & [256, 512, 1024] \\
kitchen           & [256, 512, 1024] \\
lift lowdim       & [512, 1024, 2048] \\
lift image        & [512, 1024, 2048] \\
pusht image       & [512, 1024, 2048] \\
pusht lowdim      & [512, 1024, 2048] \\
tool hang image   & [512, 1024, 2048] \\
tool hang lowdim  & [256, 512, 1024] \\
\hline
\end{tabular}
\caption{%
  Per-task base channel width $c$ and its scaled variants $2c$ and $4c$, used at
  successive downsampling stages of the Conditional 1D U-Net
  (Figure~\ref{fig:unet_arch}).
}
\label{tab:down_dims}
\end{table}

\begin{table}[!ht]
\centering
\renewcommand{\arraystretch}{1.2}

\begin{minipage}[t]{0.48\textwidth}
\centering
\begin{tabular}{l}
\hline
\textbf{ConditionalResBlock($c,\; c',\; g$)} \\
\hline
Conv1d($c \rightarrow c'$), $k{=}5$ \\
GroupNorm($c'/8$) + Mish \\
FiLM($g \rightarrow c'$) \\
Conv1d($c' \rightarrow c'$), $k{=}5$ \\
GroupNorm($c'/8$) + Mish \\
SkipConnection($c \rightarrow c'$) \\
\hline
\textbf{FiLM conditioning signal:} \\
$g = \phi(t) \;\Vert\; g_{\text{global}}$ \\
Linear $\rightarrow (scale,\; bias)$ \\
\hline
\end{tabular}
\label{fig:resblock}
\end{minipage}
\hfill
\begin{minipage}[t]{0.48\textwidth}
\centering
\begin{tabular}{lc}
\hline
\multicolumn{2}{c}{\textbf{Conditional 1D U-Net}} \\
\hline
\multicolumn{2}{l}{Input: $(B,\; T,\; d)$} \\
\hline
ResBlock($c$)          & $\times 2$ \\
Downsample($\times 2$) & \\
ResBlock($2c$)         & $\times 2$ \\
Downsample($\times 2$) & \\
ResBlock($4c$)         & $\times 2$ \\
Downsample($\times 2$) & \\
\hline
ResBlock($8c$) [bottleneck] & $\times 2$ \\
\hline
Upsample($\times 2$) & \\
ResBlock($4c$)       & $\times 2$ \\
Upsample($\times 2$) & \\
ResBlock($2c$)       & $\times 2$ \\
Upsample($\times 2$) & \\
ResBlock($c$)        & $\times 2$ \\
\hline
Conv1d($c \rightarrow d$) & \\
\hline
\multicolumn{2}{l}{Output: $(B,\; T,\; d)$} \\
\hline
\end{tabular}
\label{fig:unet_arch-2}
\end{minipage}
\caption{ Conditional U-Net backbone architecture and the corresponding residual block for the Robotics
tasks}
\label{fig:unet_arch}

\end{table}

\subsection{Model and Training Setup}
\paragraph{Model Architecture.}
We adopt the Conditional 1D U-Net from
Diffusion Policy~\cite{chi2024diffusionpolicy}, matching its architecture
and capacity to enable a controlled comparison. Tables in~\ref{fig:unet_arch} show the architecture.

\paragraph{Training Setup}
All models were trained with AdamW and a variety of learning rate schedules following~\cite{chi2024diffusionpolicy}. Performance is evaluated by average success rate per task. 

\textbf{Resource Usage and Time}. We run all experiments on a single RTX 3090 GPU. Each training run requires approximately 12 GiB of GPU memory less than 24 hours. With the exception of the blockpush task which took 48 hours.

\subsubsection{Ablation Study on Parameterizations}
\label{appendix:robot-ablations}
We evaluate three initialization strategies for the initial velocity $v^{(i)}_0$:
\begin{itemize}
    \item \textbf{Learned Initialization:} A secondary network predicts $v^{(i)}_0$ conditioned on the state. This improves peak performance but increases computational overhead.
    
    \item \textbf{Zero Initialization:} We set $v^{(i)}_0 = 0$ at inference time. This yields strong performance while reducing training complexity and is used as the default.

    \item \textbf{Historical Warm-starting:} We initialize $v^{(i)}_0$ using the velocity from the previous optimization step. This consistently leads to instability and degraded performance.
\end{itemize}

We hypothesize that the failure of historical warm-starting is due to non-stationarity in the optimization landscape, where evolving model parameters render past velocities stale.

We observe slower convergence on PushT and BlockPush compared to other tasks. We attribute this to the contact-rich and highly multimodal nature of these environments, where small action errors can lead to large deviations in long-term object motion.

\end{document}

%% file: math_commands.tex
\usepackage{amsmath,amsfonts,bm}

\def\eqref#1{equation~\ref{#1}}
\def\Eqref#1{Equation~\ref{#1}}

\def\1{\bm{1}}

\def\vx{{\bm{x}}}

\DeclareMathAlphabet{\mathsfit}{\encodingdefault}{\sfdefault}{m}{sl}
\SetMathAlphabet{\mathsfit}{bold}{\encodingdefault}{\sfdefault}{bx}{n}

\def\gL{{\mathcal{L}}}

\def\gN{{\mathcal{N}}}

